\documentclass[conference]{IEEEtran}
\IEEEoverridecommandlockouts
\usepackage{cite}
\usepackage{amsmath,amssymb,amsfonts}
\usepackage{algorithm}
\usepackage{algorithm}         
\usepackage{algpseudocode} 
\usepackage{graphicx}
\usepackage{textcomp}
\usepackage{xcolor}
\usepackage{booktabs}
\usepackage{siunitx}
\usepackage{svg}
\usepackage{graphicx}
\usepackage[caption=false,font=footnotesize]{subfig}
\usepackage[nomain, acronym, toc, shortcuts]{glossaries}
\usepackage{booktabs}
\usepackage{siunitx}
\usepackage{amsmath}
\usepackage{amsthm}

\newacronym{dml}{DML}{Distributed Machine Learning}
\newacronym{har}{HAR}{Human Activity Recognition}
\newacronym{adl}{ADL}{Activities of Daily Living}
\newacronym{fl}{FL}{Federated Learning}
\newacronym{sl}{SL}{Split Learning}
\newacronym{p2p}{P2P}{Peer-to-Peer}
\newacronym{tl}{TL}{Transfer Learning}
\newacronym{sci}{SCI}{Spinal Cord Injury}
\newacronym{shc}{SHC}{Secondary Health Condition}
\newacronym{shcs}{SHCs}{Secondary Health Conditions}

\newtheorem{theorem}{Theorem}
\newtheorem{lemma}[theorem]{Lemma}

\makeglossaries
\def\BibTeX{{\rm B\kern-.05em{\sc i\kern-.025em b}\kern-.08em
    T\kern-.1667em\lower.7ex\hbox{E}\kern-.125emX}}
\begin{document}

\title{FedSCS-XGB - Federated Server-centric surrogate XGBoost for continual health monitoring}

\author{
\IEEEauthorblockN{Felix Walger}
\IEEEauthorblockA{\textit{Chair of Information Theory and Data Analytics (INDA)} \\
\textit{RWTH Aachen University} \\
Aachen, Germany \\
felix.walger@inda.rwth-aachen.de}
\and
\IEEEauthorblockN{Mehdi Ejtehadi}
\IEEEauthorblockA{\textit{Spinal Cord Injury Artificial Intelligence (SCAI) Lab} \\
\textit{ETH Zürich \& Swiss Paraplegic Research (SPF)} \\
Nottwil, Switzerland \\
mehdi.ejtehadi@hest.ethz.ch}
\and
\IEEEauthorblockN{Anke Schmeink}
\IEEEauthorblockA{\textit{Chair of Information Theory and Data Analytics (INDA)} \\
\textit{RWTH Aachen University} \\
Aachen, Germany \\
anke.schmeink@inda.rwth-aachen.de}
\and
\IEEEauthorblockN{Diego Paez-Granados}
\IEEEauthorblockA{\textit{Spinal Cord Injury Artificial Intelligence (SCAI) Lab} \\
\textit{ETH Zürich \& Swiss Paraplegic Research (SPF)} \\
Nottwil, Switzerland \\
diego.paez@hest.ethz.ch}
 }

\maketitle

\begin{abstract}
Wearable sensors with local data processing can detect health threats early, enhance documentation, and support personalized therapy. In the context of spinal cord injury (SCI), which involves risks such as pressure injuries and blood pressure instability, continuous monitoring can help mitigate these by enabling early deDtection and intervention.
In this work, we present a novel distributed machine learning (DML) protocol for human activity recognition (HAR) from wearable sensor data based on gradient-boosted decision trees (XGBoost). The proposed architecture is inspired by Party-Adaptive XGBoost (PAX) while explicitly preserving key structural and optimization properties of standard XGBoost, including histogram-based split construction and tree-ensemble dynamics. 
First, we provide a theoretical analysis showing that, under appropriate data conditions and suitable hyperparameter selection, the proposed distributed protocol can converge to solutions equivalent to centralized XGBoost training. 
Second, the protocol is empirically evaluated on a representative wearable-sensor HAR dataset, reflecting the heterogeneity and data fragmentation typical of remote monitoring scenarios. Benchmarking against centralized XGBoost and IBM PAX demonstrates that the theoretical convergence properties are reflected in practice. The results indicate that the proposed approach can match centralized performance up to a gap under 1\% while retaining the structural advantages of XGBoost in distributed wearable-based HAR settings.
\end{abstract}


\section{Introduction}
Many chronic health conditions face a lifelong and evolving risk of \ac{shcs}. Individuals with \ac{sci} live with chronic pain, pressure injuries, and cardiometabolic complications, which frequently intensify after discharge into community living~\cite{bresnahan2022pain,glisic2024shc}. 
These risks exhibit pronounced inter-individual variability and long-term temporal dynamics driven by injury characteristics, aging, and everyday behavior, rendering episodic clinical assessment insufficient.
Wearable sensing enables continuous observation of functional and physiological patterns in daily life, but resulting data streams are inherently heterogeneous, non-stationary, and privacy-sensitive. 
\ac{har} models trained on such data implicitly encode sensitive behavioral and health-related information \cite{Ejtehadi2023LearningRehabilitation, bensland2023}, which cannot be reliably disentangled once transmitted to a central entity~\cite{elgendi2025balancing}. 
Distributed Machine Learning (DML) addresses this challenge by enabling collaborative model training while retaining raw sensor data locally, thereby supporting personalization under non-IID data and limiting information leakage~\cite{vepakomma_split_2018,kalabakov2024federated}.
Gradient-boosted decision trees such as XGBoost \cite{chen2016xgboost} provide deterministic inference, predictable computational complexity, and interpretable decision logic through tree splits and additive structure \cite{friedman2000greedy,molnar2019interpretable}. These properties make them attractive for edge-based biomedical HAR, but standard XGBoost training assumes centralized access to the full dataset, which conflicts with privacy-preserving wearable deployments.
To bridge this gap, we propose a foundational DML protocol for XGBoost-based wearable monitoring for chronic lifelong conditions such as \ac{sci}, conceptually shown in Fig. \ref{fig:continual_monitoring}. The protocol enables distributed, privacy-aware training while preserving the key mathematical and optimization behavior of centralized, histogram-based boosting. Inspired by existing work \cite{ong2020adaptive}, our approach places stronger emphasis on retaining native histogram construction and split-finding dynamics, rather than adapting model structure to protocol constraints. 

\begin{figure}
    \centering
    \includegraphics[width=0.7\linewidth]{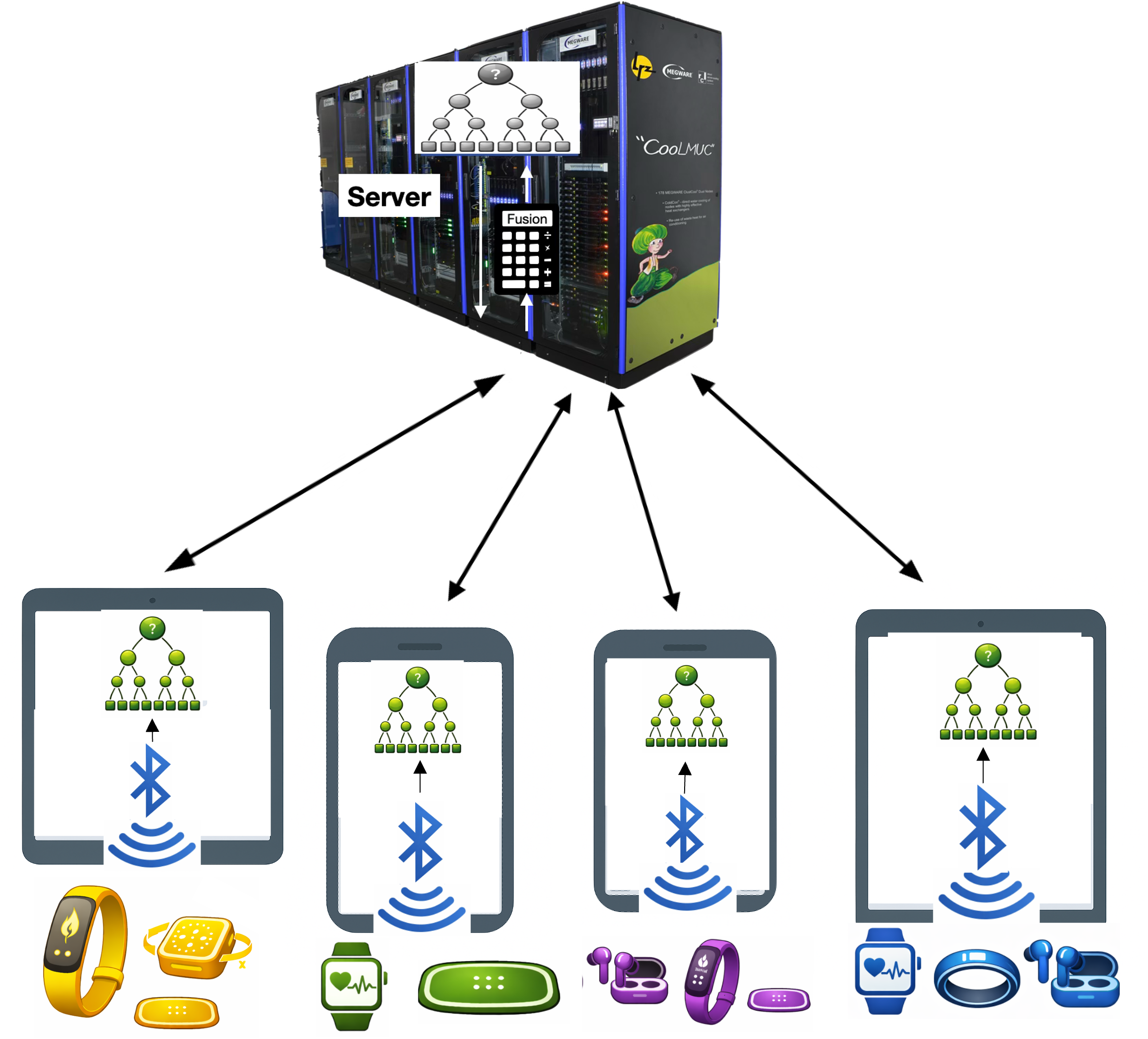}
    \caption{A continual health monitoring system with \ac{fl}.}
    \label{fig:continual_monitoring}
\end{figure}
The main contributions of this work are:
\begin{enumerate}
    \item an analysis of fundamental requirements for distributed machine learning in lifelong monitoring - focused on chronic health conditions,
    \item design of a PAX-inspired distributed XGBoost protocol, preserving core boosting and histogram mechanisms,
    \item a convergence analysis demonstrating equivalence to centralized XGBoost under suitable conditions, and
    \item an evaluation on feature-extracted \ac{har} data based on a emulated implementation using Flower \cite{beutel_flower_2022}.
\end{enumerate}

\section{Related Work}
\label{sec:related}


\ac{fl} is a central paradigm for privacy-preserving model training across distributed wearable and mobile sensors and has been widely adopted for \ac{har}. By enabling collaborative learning without centralizing raw sensor data, FL directly addresses privacy and data governance concerns \cite{aouedi2024flsurvey}. 
Early feasibility studies~\cite{sozinov2018flhar} showed that baseline algorithms such as \textit{FedAvg} can achieve competitive performance relative to centralized training, while also revealing sensitivity to client heterogeneity and data imbalance. Subsequent work therefore introduced architectural extensions to improve robustness and scalability, including hierarchical FL~\cite{concone2022hierarchicalfl}, hybrid horizontal--vertical federation~\cite{zhou2022twoDfl}, and cluster-based aggregation strategies~\cite{ouyang2023clusterfl}.

Personalization has emerged as a dominant research direction in FL-based HAR. Approaches such as \textit{FedHealth}~\cite{chen2020fedhealth}, \textit{Proto-HAR}~\cite{cheng2023protohar}, and meta-learning-based methods~\cite{li2021metahar} explicitly separate shared representations from user-specific adaptation, consistently outperforming purely global models under heterogeneous client distributions.

Recent studies address deployment-related constraints on heterogeneous edge devices. Adaptive systems such as \textit{Hydra}~\cite{wang2024hydra} and knowledge-distillation-based FL methods~\cite{tu2021feddl} reduce communication and computation overhead. Complementary system-level analyses~\cite{gajanin2024asynchronous} investigate robustness under bandwidth limitations and asynchronous updates, while additional work focuses on privacy compliance through encryption and machine unlearning mechanisms~\cite{xiao2021harflprivacy}.
However, Most existing \ac{fl} research in HAR focuses on neural network architectures, while federated ensemble methods remain comparatively underexplored. For horizontal FL settings several Federated XGBoost approaches have been proposed.

A foundational contribution is \emph{Party-Adaptive XGBoost} (PAX)~\cite{ong2020adaptive}, which extends XGBoost’s histogram-based split finding using client-specific quantile sketches. These sketches are aggregated at the server to construct globally consistent split candidates without exposing raw data. Subsequent studies~\cite{jones_federated_2022} demonstrate that Federated XGBoost maintains performance close to centralized training even under severe sample-wise non-IID partitions, highlighting the inherent robustness of tree-based models. 

To further reduce communication overhead, ~\cite{ma_gradient-less_2023} propose a gradient-less Federated XGBoost approach that aggregates locally trained ensembles instead of gradients or histograms. Alternative work explores improved sampling strategies, such as Minimal Variance Sampling~\cite{lindskog_histogram-based_2023}, to enhance split quality under limited communication. A recent survey~\cite{bodynek_applying_2023} categorizes these methods into histogram-based, ensemble-aggregation, and sampling-enhanced approaches, outlining their respective trade-offs.
Despite substantial progress, current distributed learning approaches for HAR exhibit several limitations. Many frameworks are highly scenario-specific and rely on deep neural networks that are computationally demanding for long-term edge deployment~\cite{li2021metahar}. Moreover, ensemble-based methods such as XGBoost remain underrepresented in federated HAR, despite their advantages in interpretability, energy efficiency, and robustness to heterogeneous data.
Motivated by these gaps, this work investigates a communication-efficient and privacy-preserving Federated XGBoost framework tailored to heterogeneous wearable sensor data, with a focus on scalable and interpretable activity recognition in real-world rehabilitation scenarios.

\section{Methods}




\subsection{Sensory System, ADLs, and Data Processing}
\label{subsec:sensors}

Based on the prior validated ADL recognition framework in \cite{Ejtehadi2023LearningRehabilitation}, we consider \textbf{16} clinically relevant ADLs/states for classification using the unobtrusive multimodal sensor configuration in Fig.~\ref{fig:sensors}, covering wheelchair activities (resting, phone/computer use, arm raises, eating/drinking, hand cycling), mobility modes (assisted/self propulsion), transfer events (to/from wheelchair), pressure relief, four lying postures (back/right/left/stomach), and the empty-wheelchair state.

The time series data were processed using a 5th-order Butterworth low-pass filter with a cut-off frequency of 5~Hz (10~Hz for the wheel IMU) \cite{Ejtehadi2023LearningRehabilitation}. The filtered signals were resampled to 20~Hz to reduce computational complexity and segmented using a sliding-window approach; based on preliminary experiments, a 4~s window with a 2~s step (50\% overlap) was used. Overall, data from 8 individuals were included, yielding 44{,}358 windows across the 16 classes, with moderate class imbalance: phone/computer/eating typically contributed $\sim$600--700 windows per subject, posture and propulsion-related classes were commonly around $\sim$300 windows per subject, and transfer activities were comparatively sparse (tens to a few hundred windows per subject).

For feature extraction, we used the TIFEX-Py toolbox
~\cite{tifexpy} to compute a compact set of time- and frequency-domain descriptors per window, including first- and second-order statistics (mean, standard deviation, variance, median, min/max), higher-order moments (skewness, kurtosis), change/complexity measures (absolute sum of changes, mean change, mean absolute change, longest strike above/below mean, sample entropy, CID-CE), and spectral/periodicity features (FFT coefficient magnitude, Welch density, autocorrelation, number of peaks, and zero-crossing rate).

\subsection{FedSCS-XGB}
We introduce here a PAX-inspired Federated implementation of XGBoost called Federated Server-Centric Surrogate XGBoost.
\begin{figure}[t]
  \centering
  \includegraphics[width=\columnwidth]{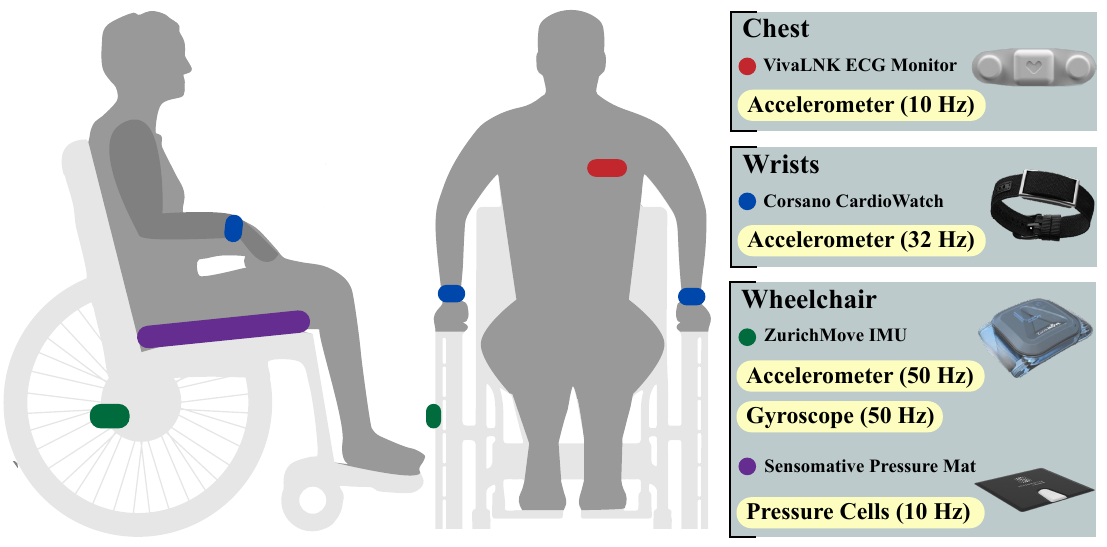}
  \caption{Sensor setup for the SCI wheelchair user: VivaLNK Wearable ECG monitor sensor on chest, two Corsano’s CardioWatch 287-2 worn at wrists, Sensomative wheelchair mat placed under bottom cushion, and Zurichmove IMU on the wheel of the wheelchair.}
  \label{fig:sensors}
\end{figure}
\paragraph{Protocol overview}
FedSCS-XGB alternates two phases. In sketch rounds, clients transmit mergeable per-feature DDSketches weighted by (diagonal) Hessians; the server merges these sketches and derives global bin edges.
In atom rounds, clients quantize samples using the received edges, aggregate sufficient statistics $(W,G,H)$ per multi-feature bin vector (atom key), and transmit these atom summaries.
The server conditions atoms on the current node-path, evaluates histogram split candidates, selects the best split per node, and returns the resulting stump(s) to clients to update local margins. The protocol procedures are illustrated in Algorithms \ref{alg:scs_client_1col} and \ref{alg:scs_server_1col}.

\paragraph{Multiclass softmax and tree groups.}
For multiclass classification with $K$ classes, the client computes softmax probabilities and per-class gradients/Hessians.
The server grows one tree per class per boosting round (XGBoost \emph{tree group} semantics), i.e., it evaluates gains per class and sums them to select splits; leaf weights are class-wise vectors.

\paragraph{Relation to PAX}
Fed SCS-XGB is inspired by the Party-Adaptive XGBoost (PAX) framework of Ong \emph{et al} \cite{ong_adaptive_2020}, but differs in several key aspects.
While PAX constructs party-specific surrogate \emph{feature representations} and associated histogram statistics to guide split evaluation,
Fed SCS-XGB explicitly separates bin-edge estimation and sufficient-statistic aggregation into two synchronous communication phases.
In particular, global bin edges are determined via Hessian-weighted, mergeable sketches, after which clients transmit only aggregated atom statistics per multi-feature bin vector.
Furthermore, Fed SCS-XGB performs local prediction and gradient computation on raw feature values and uses the surrogate bins solely for aggregation,
whereas PAX conceptually evaluates split candidates with respect to the surrogate feature representation.
These modifications preserve the underlying XGBoost objective while improving communication efficiency and simplifying synchronization,
and they form the basis for the convergence analysis presented in this work.

\begin{algorithm}[t]
\caption{FedSCS-XGB Server (multiclass: Sketch $\rightarrow$ Atoms)}
\label{alg:scs_server_1col}
\begin{algorithmic}[1]
\Require Rounds $T$, features $d$, bins $B$, depth $D$, reg.\ $\lambda,\gamma$
\State Initialize global additive model $f^{(A)}\gets 0$
\For{$t=1,\dots,T$}
  \State Sample clients $C_t$

  \Statex \hspace{-0.6em}\textbf{Phase 1 (Sketch).} \Comment{\textbf{Comm:} S$\rightarrow$C}
  \State Send \texttt{SKETCH\_REQ}$(t,B)$ to all $c\in C_t$
  \State \textbf{Barrier:} await $\{\mathcal{K}_{c,f}^{(t)}\}_{f=1}^d$ from all $c\in C_t$
  \For{$f=1,\dots,d$}
    \State $\mathcal{E}_f^{(t)} \gets \textsc{Quantiles}\!\big(\textsc{Merge}(\{\mathcal{K}_{c,f}^{(t)}\}_{c\in C_t}),\,B\big)$
  \EndFor

  \Statex \hspace{-0.6em}\textbf{Phase 2 (Atoms).} \Comment{\textbf{Comm:} S$\rightarrow$C}
  \State Send \texttt{ATOM\_REQ}$(t,f^{(A)}_{t-1},\mathcal{E}^{(t)})$ to all $c\in C_t$
  \State \textbf{Barrier:} await atom maps $\{(W_r^{(c)},G_{r,\cdot}^{(c)},H_{r,\cdot}^{(c)})\}_{r\in\mathcal{R}_c}$ from all $c\in C_t$
  \State Merge atoms by key: $(W_r,G_{r,\cdot},H_{r,\cdot}) \gets \sum_{c\in C_t}(W_r^{(c)},G_{r,\cdot}^{(c)},H_{r,\cdot}^{(c)})$

  \State Grow trees depth-wise up to $D$ using histogram split search on atoms:
  \For{each active node $a$}
    \For{each $(f,q)$ with $q\in\{1,\dots,B-1\}$}
      \State $(G_{L,\cdot},H_{L,\cdot},G_{R,\cdot},H_{R,\cdot}) \gets \textsc{PrefixStats}(a,f,q;\{(W_r,G_{r,\cdot},H_{r,\cdot})\})$
      \State $\mathrm{Gain}(a,f,q)\gets \sum_{k=1}^K \mathrm{Gain}(G_{L,k},H_{L,k},G_{R,k},H_{R,k})$ using Eq.~\eqref{eq:gain_precise}
    \EndFor
    \State $(f^\star,q^\star)\gets \arg\max_{f,q}\mathrm{Gain}(a,f,q)$ \Comment{deterministic tie-break}
    \State $w_{L,k}\gets -G_{L,k}/(H_{L,k}+\lambda)$,\;\; $w_{R,k}\gets -G_{R,k}/(H_{R,k}+\lambda)$
    \State Add split $(f^\star,\theta=\mathcal{E}^{(t)}_{f^\star}[q^\star],w_L,w_R)$ to $\{T_{t,k}\}_{k=1}^K$
  \EndFor
  \State $f^{(A)}_{t}\gets f^{(A)}_{t-1} + \eta\cdot\{T_{t,k}\}_{k=1}^K$
\EndFor
\end{algorithmic}
\end{algorithm}

\begin{algorithm}[t]
\caption{FedSCS-XGB Client $p_i$ (multiclass softmax)}
\label{alg:scs_client_1col}
\begin{algorithmic}[1]
\Require Local data $(X,y)$, $N_i$ samples, $K$ classes, $d$ features
\State Maintain margins $\mathbf{m}\in\mathbb{R}^{N_i\times K}$ (init.\ $0$)
\While{requests arrive}
  \State \textbf{Receive} message from server
  \If{\texttt{SKETCH\_REQ}$(t,B)$}
    \State Compute per-sample $(g_{i,k}^{(t)},h_{i,k}^{(t)})$ (softmax, diag.\ Hessian) and weights $w_i^{(t)}\gets \sum_{k=1}^K h_{i,k}^{(t)}$
    \For{$f=1,\dots,d$}
      \State Build weighted sketch $\mathcal{K}_{i,f}^{(t)}$ from $\{(x_{i,f},w_i^{(t)})\}_{i=1}^{N_i}$
    \EndFor
    \State \textbf{Send} $\{\mathcal{K}_{i,f}^{(t)}\}_{f=1}^d$
  \ElsIf{\texttt{ATOM\_REQ}$(t,f^{(A)}_{t-1},\mathcal{E}^{(t)})$}
    \State Update margins: $\mathbf{m}\gets \mathbf{m}+\eta\cdot f^{(A)}_{t-1}(X)$
    \State Compute per-sample $(g_{i,\cdot}^{(t)},h_{i,\cdot}^{(t)})$ as above
    \State Bin-ID vectors: $\mathbf{b}_i\gets ( \textsc{BinIndex}(\mathcal{E}^{(t)}_f,x_{i,f}) )_{f=1}^d$
    \State Aggregate atoms per occupied key $r$: $(W_r,G_{r,\cdot}^{(t)},H_{r,\cdot}^{(t)}) \gets \sum_{i:\mathbf{b}_i=r}(1,g_{i,\cdot}^{(t)},h_{i,\cdot}^{(t)})$
    \State \textbf{Send} $\{(W_r,G_{r,\cdot}^{(t)},H_{r,\cdot}^{(t)})\}_{r\in\mathcal{R}_i}$
  \EndIf
\EndWhile
\end{algorithmic}
\end{algorithm}

\subsection{Convergence to the Centralized Histogram XGBoost Objective
via Hessian-Weighted DDSketch Binning (SCS-XGB)}
\label{subsec:scs_xgb_convergence_precise}

\paragraph{Notation}
Table~\ref{tab:notation_scsxgb_precise} summarizes all quantities used throughout this subsection.

\begin{table}[t]
\caption{Notation for the SCS-XGB convergence analysis.}
\centering
\footnotesize
\setlength{\tabcolsep}{4pt}
\renewcommand{\arraystretch}{1.05}
\begin{tabular}{ll}
\toprule
Symbol & Meaning \\
\midrule
$\mathcal{D}=\{(x_j,y_j)\}_{j=1}^N$ & Dataset, $N$ samples \\
$x_j\in\mathbb{R}^d$ & Feature vector, $d$ features \\
$M_t$ & Model after $t$ boosting rounds \\
$T,m$ & \# rounds, max.\ tree depth \\[0.2em]

$\ell(y,\hat y)$ & Loss, convex and $C^2$ in $\hat y$ \\
$g_j^{(t)},\,h_j^{(t)}$ & First/second derivatives of $\ell$ \\[0.2em]

$\lambda,\gamma$ & Regularization, split penalty \\[0.2em]

$B$ & Bins per feature \\
$f\in\{1,\dots,d\}$ & Feature index \\[0.2em]

$F_f^{(t)}$ & Hessian-weighted empirical CDF \\
$E_f^{(t)},\,\tilde E_f^{(t)}$ & Ideal / sketch histogram edges \\
$\alpha$ & CDF (rank) approximation error \\[0.2em]

$\mathbf{b}(x)$ & Bin-ID vector induced by $\tilde E^{(t)}$ \\
$r$ & Atom (bin-ID) key \\
$W_r,\,G_r^{(t)},\,H_r^{(t)}$ & Atom count, gradient, Hessian \\[0.2em]

$(G_L,H_L),(G_R,H_R)$ & Split sufficient statistics \\[0.2em]

$\mathcal{A}_t$ & Nodes evaluated at round $t$ \\
$\mathrm{Gain}$ & Split gain (Eq.~\eqref{eq:gain_precise}) \\[0.2em]

$J_t^{\mathrm{cent}},\,J_t^{\mathrm{scs}}$ & Centralized / SCS-XGB objective \\
\bottomrule
\end{tabular}
\label{tab:notation_scsxgb_precise}
\end{table}

\paragraph{Second-order gain (XGBoost)}
For a candidate split of a node into left/right parts with sufficient statistics
$(G_L,H_L)$ and $(G_R,H_R)$, the XGBoost gain is
\begin{equation}
\begin{aligned}
\mathrm{Gain}(G_L,H_L,G_R,H_R)
:=\;&
\frac12\Bigg(
\frac{G_L^2}{H_L+\lambda}
+\frac{G_R^2}{H_R+\lambda}
\\[-0.35em]
&\quad
-\frac{(G_L+G_R)^2}{(H_L+H_R)+\lambda}
\Bigg)
-\gamma .
\end{aligned}
\label{eq:gain_precise}
\end{equation}

\paragraph{Hessian-weighted quantiles.}
At round $t$, define the Hessian-weighted empirical CDF for feature $f$ as
\begin{equation}
F_f^{(t)}(v)
:=\frac{1}{\sum_{j=1}^N h_j^{(t)}}
\sum_{j:\,x_j^{(f)}\le v} h_j^{(t)} .
\label{eq:weighted_cdf}
\end{equation}
We define ideal histogram edges by the left-quantile convention
\begin{equation}
e_{f,b}^{(t)}
:= \inf\Big\{v:\;F_f^{(t)}(v)\ge b/B\Big\},
\qquad b=0,\dots,B .
\label{eq:left_quantile_def}
\end{equation}
SCS-XGB uses mergeable Hessian-weighted DDSketches to obtain approximate edges
$\tilde E_f^{(t)}=(\tilde e_{f,0}^{(t)}<\cdots<\tilde e_{f,B}^{(t)})$
\cite{masson2019ddsketch}; edges are fixed during split search.

\paragraph{Goal}
For any finite horizon $T$ and tolerance $\varepsilon>0$, we show existence of a
sketch accuracy level $\bar\alpha$ such that $\alpha\le\bar\alpha$ implies
$\lvert J_t^{\mathrm{scs}}-J_t^{\mathrm{cent}}\rvert\le\varepsilon$ for all $t\le T$.

\paragraph{Assumptions.}
\textbf{A1 (Bounded derivatives).}
$\exists\,G_{\max},H_{\max}<\infty$ such that
$|g_j^{(t)}|\le G_{\max}$ and $0\le h_j^{(t)}\le H_{\max}$ for all $j,t$.

\textbf{A2 (Bracketing accuracy).}
$\exists\,\alpha\in(0,1)$ such that for all $t,f,b$,
$$
F_f^{(t)}(\tilde e_{f,b}^{(t)}-) \le \frac{b}{B}+\alpha,
\qquad
F_f^{(t)}(\tilde e_{f,b}^{(t)}) \ge \frac{b}{B}-\alpha .
\label{eq:A2_bracket}
$$

\textbf{A3 (Fixed edges).}
Within each round $t$, the edges $\tilde E^{(t)}$ are fixed during depth-wise
split evaluation.
\textbf{A4 (Deterministic tie-breaking).}
Centralized histogram XGBoost and SCS-XGB break equal-gain ties identically.
\begin{lemma}[Exact surrogate equivalence]
\label{lem:exact_surrogate}
Fix a boosting round $t$ and sketch edges $\tilde E^{(t)}$.
Greedy split evaluation using routed atom statistics is exactly equivalent to
running histogram XGBoost on the induced atom pseudo-dataset with edges
$\tilde E^{(t)}$.
\end{lemma}
\begin{proof}
Fix $t$ and $\tilde E^{(t)}$. For any node and candidate split $(f,k)$, routing
depends only on bin indices and is therefore deterministic.
Aggregating $(G_r^{(t)},H_r^{(t)},W_r)$ over atom keys yields exactly the
left/right sufficient statistics that histogram XGBoost computes on the
pseudo-dataset with the same edges.
Since the gain~\eqref{eq:gain_precise} depends only on these statistics, all
candidate gains and the greedy split selection coincide. \qedhere
\end{proof}
\begin{lemma}[Hessian prefix-mass perturbation]
\label{lem:hessian_prefix}
Hold A1--A2. Fix $t$ and feature $f$.
Define the ideal and sketch-induced Hessian prefix masses
$$
H_{\le b}^{(t)}
:= \sum_{j:\,x_j^{(f)}\le e_{f,b}^{(t)}} h_j^{(t)},
\qquad
\tilde H_{\le b}^{(t)}
:= \sum_{j:\,x_j^{(f)}\le \tilde e_{f,b}^{(t)}} h_j^{(t)} .
$$
Then, for all $b\in\{0,\dots,B\}$,
\[
\big|\tilde H_{\le b}^{(t)}-H_{\le b}^{(t)}\big|
\le \alpha \sum_{j=1}^N h_j^{(t)} .
\]
\end{lemma}

\begin{proof}
By Definition~\eqref{eq:weighted_cdf},
\[
\tilde H_{\le b}^{(t)}
= \Big(\sum_j h_j^{(t)}\Big) F_f^{(t)}(\tilde e_{f,b}^{(t)}),
\quad
H_{\le b}^{(t)}
= \Big(\sum_j h_j^{(t)}\Big) F_f^{(t)}(e_{f,b}^{(t)}).
\]
Under the left-quantile convention~\eqref{eq:left_quantile_def},
$F_f^{(t)}(e_{f,b}^{(t)}) \ge b/B$ and
$F_f^{(t)}(e_{f,b}^{(t)}-) < b/B$.
Combining this with the bracketing condition \emph{A2} yields
\[
\big|F_f^{(t)}(\tilde e_{f,b}^{(t)})
      -F_f^{(t)}(e_{f,b}^{(t)})\big|
\le \alpha ,
\]
and multiplying by $\sum_j h_j^{(t)}$ gives the stated bound. \qedhere
\end{proof}

\paragraph{Definition (Node-wise statistic perturbation constant)}
Fix a boosting round $t$ and node $a$.
Let $\mathbf{s}_{t,a}(f,k)$ and $\tilde{\mathbf{s}}_{t,a}(f,k)$ denote the ideal
and sketch-induced sufficient statistics for candidate split $(f,k)$.
Define $K_{t,a}\ge0$ such that
\begin{equation}
\sup_{(f,k)}
\big\|\tilde{\mathbf{s}}_{t,a}(f,k)-\mathbf{s}_{t,a}(f,k)\big\|_1
\le K_{t,a}\,\alpha .
\label{eq:K_def}
\end{equation}
Such a finite constant exists for fixed data, $t$, and $a$ whenever statistic
perturbations scale linearly in ~$\alpha$ for sufficiently small ~$\alpha$.


\begin{lemma}[Lipschitz continuity of the gain]
\label{lem:gain_lipschitz}
Fix $\lambda>0$.
The gain functional~\eqref{eq:gain_precise} is Lipschitz continuous on any bounded
domain $|G|\le\bar G$, $0\le H\le\bar H$, with
\[
\big|\mathrm{Gain}(\mathbf{s})-\mathrm{Gain}(\mathbf{s}')\big|
\le L(\bar G,\bar H,\lambda)\,
\|\mathbf{s}-\mathbf{s}'\|_1 .
\]
\end{lemma}

\begin{proof}
On the bounded domain, all denominators are bounded below by $\lambda>0$, hence all
partial derivatives are bounded. The result follows from the mean value theorem.
\qedhere
\end{proof}

\paragraph{Definition (Node-wise gain sensitivity).}
For each $(t,a)$ define
\[
C_{t,a}
:= L(\bar G_{t,a},\bar H_{t,a},\lambda)\,K_{t,a},
\]
where $\bar G_{t,a}$ and $\bar H_{t,a}$ bound all candidate split statistics at
node $a$.


\begin{lemma}[Uniform gain perturbation]
\label{lem:uniform_gain_perturb}
For all candidate splits $(f,k)$ at node $a$,
\[
\big|
\mathrm{Gain}(\tilde{\mathbf{s}}_{t,a}(f,k))
-
\mathrm{Gain}(\mathbf{s}_{t,a}(f,k))
\big|
\le C_{t,a}\,\alpha .
\]
\end{lemma}


\begin{theorem}[Finite-horizon $\varepsilon$-approximation]
\label{thm:finite_eps}
Hold A1--A4. For any $\varepsilon>0$ and finite horizon $T$, there exists
$\bar\alpha>0$ such that for all $\alpha\le\bar\alpha$ and all $t\le T$,
\[
\big|J_t^{\mathrm{scs}}-J_t^{\mathrm{cent}}\big|\le \varepsilon .
\]
\end{theorem}

\begin{proof}
For node $a$ at round $t$, uniform gain perturbation implies ideal gain
suboptimality at most $2C_{t,a}\alpha$.
Since each split reduces the objective by exactly its gain, summing over nodes
yields
\[
\big|J_t^{\mathrm{scs}}-J_t^{\mathrm{cent}}\big|
\le \sum_{a\in\mathcal{A}_t} 2C_{t,a}\alpha
=: C_t\alpha .
\]
With $C:=\max_{t\le T} C_t<\infty$, choosing $\bar\alpha=\varepsilon/C$ proves the
claim. \qedhere
\end{proof}

\paragraph{Interpretation}
Assumption~A2 can be satisfied for arbitrarily small $\alpha$ by choosing
sufficiently accurate DDSketch hyperparameters.
Thus, for any fixed finite horizon $T$, SCS-XGB converges to the centralized
histogram XGBoost solution in objective value as $\alpha\to0$, independently of
the underlying data distribution.

\section{Results}\label{sec:results}
\subsection{Implementation}
To obtain a realistic and functionally faithful emulation of a deployable \ac{dml} system, we implemented \emph{FedSCS-XGB} on top of the Flower federated learning framework~\cite{beutel_flower_2022} in Python. We rely on Flower’s standard \texttt{server\_app} / \texttt{client\_app} architecture together with its built-in simulation capabilities, which enable controlled experimentation under reproducible client participation and communication patterns.

As a centralized baseline, we use the reference Python implementation of XGBoost~\cite{chen2016xgboost}. In addition, we implemented the Party-Adaptive XGBoost (PAX) protocol according to the pseudo-algorithm and design principles described in the original work~\cite{ong_adaptive_2020}, again using Flower as the orchestration layer.

A key advantage of Flower is its high degree of modularity. This allowed us to clearly separate (i) the distributed machine learning logic, (ii) the data processing pipeline, (iii) the protocol-level communication, control, and orchestration mechanisms, and (iv) the simulation environment. As a result, the implementation remains close to a real world deployment setting while enabling systematic comparison of different distributed learning protocols under identical experimental conditions.

\subsection{Experiments}

\begin{table*}[t]
\centering
\footnotesize
\caption{Client-wise accuracy and F1-score for centralized XGBoost (baseline),
FedSCS-XGB (scs), and PAX across histogram bin counts.}
\label{tab:pax:massive_bins_metrics}

{\setlength{\tabcolsep}{15pt}%
\renewcommand{\arraystretch}{0.9}%
\setlength{\aboverulesep}{0.2ex}%
\setlength{\belowrulesep}{0.2ex}%
\begin{tabular*}{\textwidth}{@{\extracolsep{1.2pt}}ll S S S S S S S S}
\toprule
 &  &
 \multicolumn{8}{c}{Histogram bins} \\
\cmidrule(lr){3-10}
Client & Method
& \multicolumn{2}{c}{64}
& \multicolumn{2}{c}{128}
& \multicolumn{2}{c}{256}
& \multicolumn{2}{c}{512} \\
\cmidrule(lr){3-4}\cmidrule(lr){5-6}\cmidrule(lr){7-8}\cmidrule(lr){9-10}
  &  &
 \multicolumn{1}{c}{Acc.} & \multicolumn{1}{c}{F1}
 & \multicolumn{1}{c}{Acc.} & \multicolumn{1}{c}{F1}
 & \multicolumn{1}{c}{Acc.} & \multicolumn{1}{c}{F1}
 & \multicolumn{1}{c}{Acc.} & \multicolumn{1}{c}{F1} \\
\midrule

patient01 & baseline & 92.66 & 89.86 & 93.01 & 90.10 & 93.28 & 90.90 & 93.10 & 90.39 \\
          & scs      & 90.83 & 87.34 & 91.18 & 86.70 & 91.44 & 88.54 & 91.44 & 88.54 \\
          & pax      & 89.26 & 86.14 & 89.26 & 86.14 & 89.35 & 86.24 & 85.59 & 81.56 \\
\addlinespace

patient02 & baseline & 94.34 & 91.46 & 93.93 & 90.68 & 94.43 & 91.33 & 94.18 & 91.22 \\
          & scs      & \textbf{94.51} & 91.45 & \textbf{95.25} & \textbf{92.17} & \textbf{94.34} & 89.81 & 94.34 & 89.81 \\
          & pax      & 92.71 & 88.71 & 92.71 & 88.71 & 93.03 & 88.98 & 92.30 & 87.97 \\
\addlinespace

patient03 & baseline & 97.31 & 92.87 & 97.02 & 94.00 & 96.73 & 92.83 & 96.92 & 93.29 \\
          & scs      & 96.64 & 91.96 & 96.73 & 90.98 & \textbf{96.92} & 91.73 & \textbf{96.92} & 91.73 \\
          & pax      & 96.35 & 92.20 & 96.35 & 92.20 & 96.92 & 92.93 & 96.06 & 91.00 \\
\addlinespace

patient04 & baseline & 93.59 & 92.95 & 93.77 & 93.13 & 93.59 & 93.01 & 93.59 & 92.83 \\
          & scs      & 90.74 & 88.50 & 92.64 & 90.99 & 90.82 & 87.90 & 90.82 & 87.90 \\
          & pax      & 91.17 & 89.87 & 91.17 & 89.87 & 90.13 & 88.59 & 88.14 & 85.84 \\
\addlinespace

patient05 & baseline & 96.97 & 85.63 & 97.24 & 85.40 & 97.61 & 85.97 & 97.70 & 87.14 \\
          & scs      & 96.23 & 84.19 & 96.78 & 84.89 & 96.23 & 84.95 & 96.23 & 84.95 \\
          & pax      & 96.88 & 83.09 & 96.88 & 83.09 & 95.13 & 80.50 & 88.05 & 71.75 \\
\addlinespace

patient06 & baseline & 97.11 & 92.52 & 97.48 & 92.49 & 97.02 & 91.48 & 97.48 & 92.45 \\
          & scs      & 96.65 & 92.17 & 94.97 & 91.02 & 95.15 & 90.11 & 95.15 & 90.11 \\
          & pax      & 95.06 & 89.98 & 95.06 & 89.98 & 95.25 & 88.60 & 94.97 & 89.45 \\
\addlinespace

patient07 & baseline & 96.51 & 91.99 & 96.51 & 92.83 & 96.42 & 92.91 & 96.42 & 92.90 \\
          & scs      & 96.06 & 91.43 & \textbf{96.51} & 92.08 & {95.61} & 89.92 & 95.61 & 89.92 \\
          & pax      & 93.37 & 87.67 & 93.37 & 87.67 & 94.62 & 90.54 & 91.49 & 84.97 \\
\addlinespace

patient08 & baseline & 96.22 & 91.28 & 97.00 & 93.45 & 97.09 & 93.10 & 96.51 & 91.22 \\
          & scs      & 95.74 & \textbf{91.78} & 96.61 & 92.92 & 94.86 & 88.77 & 94.86 & 88.77 \\
          & pax      & 96.42 & 91.45 & 96.42 & 91.45 & 96.22 & 92.54 & 94.77 & 89.74 \\
\midrule

ALL & baseline
& 95.59$\pm$1.79 & 91.07$\pm$2.42
& 95.75$\pm$1.84 & 91.51$\pm$2.81
& 95.77$\pm$1.72 & 91.44$\pm$2.38
& 95.74$\pm$1.83 & 91.43$\pm$2.01 \\
    & scs
& 94.67$\pm$2.49 & 89.85$\pm$2.90
& 95.08$\pm$2.11 & 90.22$\pm$2.86
& 94.42$\pm$2.19 & 88.97$\pm$2.00
& 94.42$\pm$2.19 & 88.97$\pm$2.00 \\
    & pax
& 93.90$\pm$2.76 & 88.64$\pm$2.97
& 93.90$\pm$2.76 & 88.64$\pm$2.97
& 93.83$\pm$2.78 & 88.62$\pm$3.95
& 91.42$\pm$3.82 & 85.28$\pm$6.26 \\
\bottomrule
\end{tabular*}
} 
\end{table*}



\begin{figure*}[t]
    \centering
    \subfloat[Acc.\label{fig:acc_overlap}]{
        \includegraphics[width=0.48\textwidth]{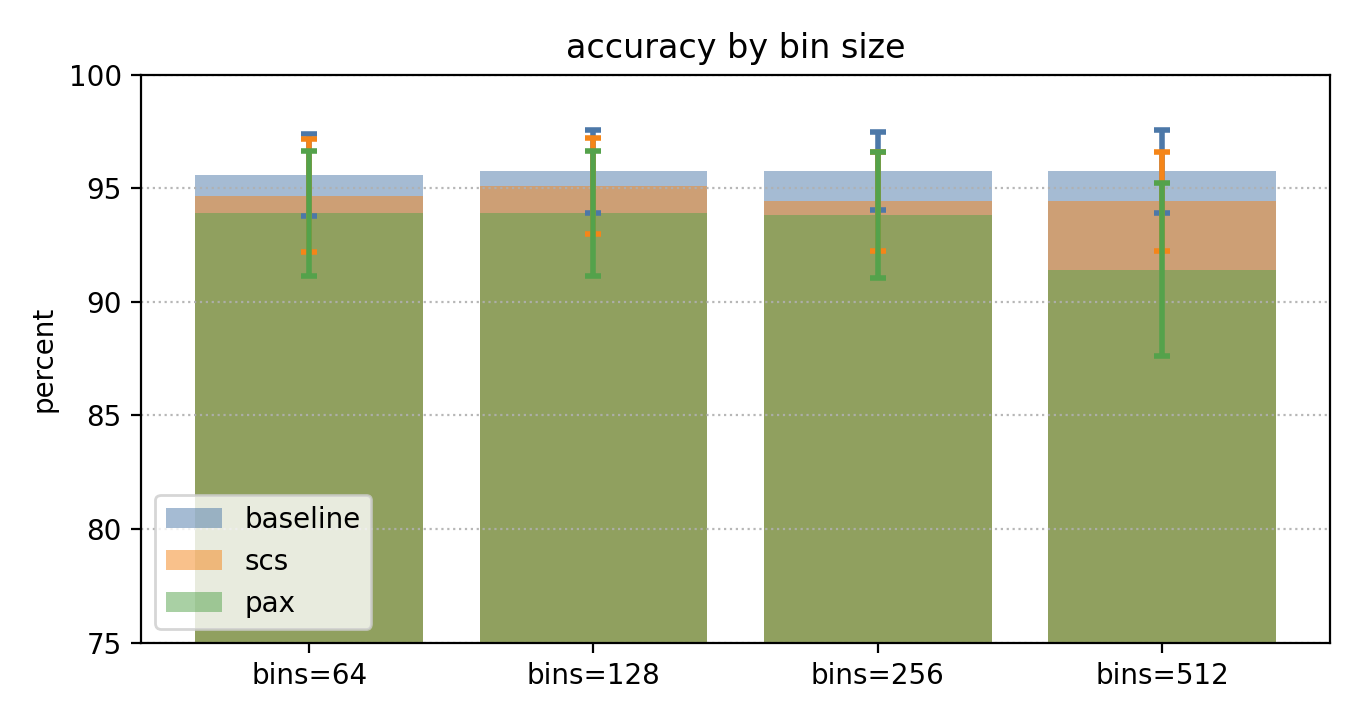}
    }
    \hfill
    \subfloat[F1.]{
        \includegraphics[width=0.48\textwidth]{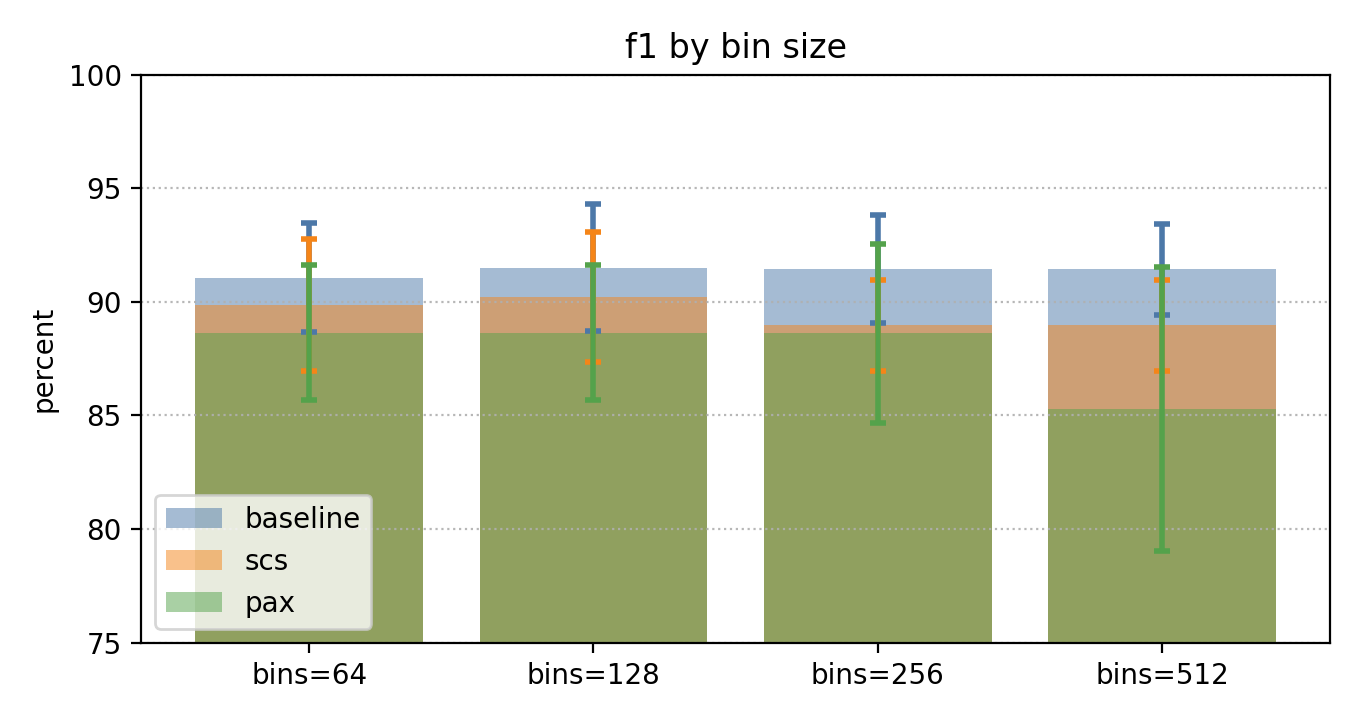}
    }
    \caption{Performance illustration of FedSCS, PAX and XGBoost-baseline for varying bin sizes.}
    \label{fig:diff_comp}
\end{figure*}

\begin{figure*}[t]
    \centering
    \subfloat[Acc.\label{fig:acc_overlap}]{
        \includegraphics[width=0.48\textwidth]{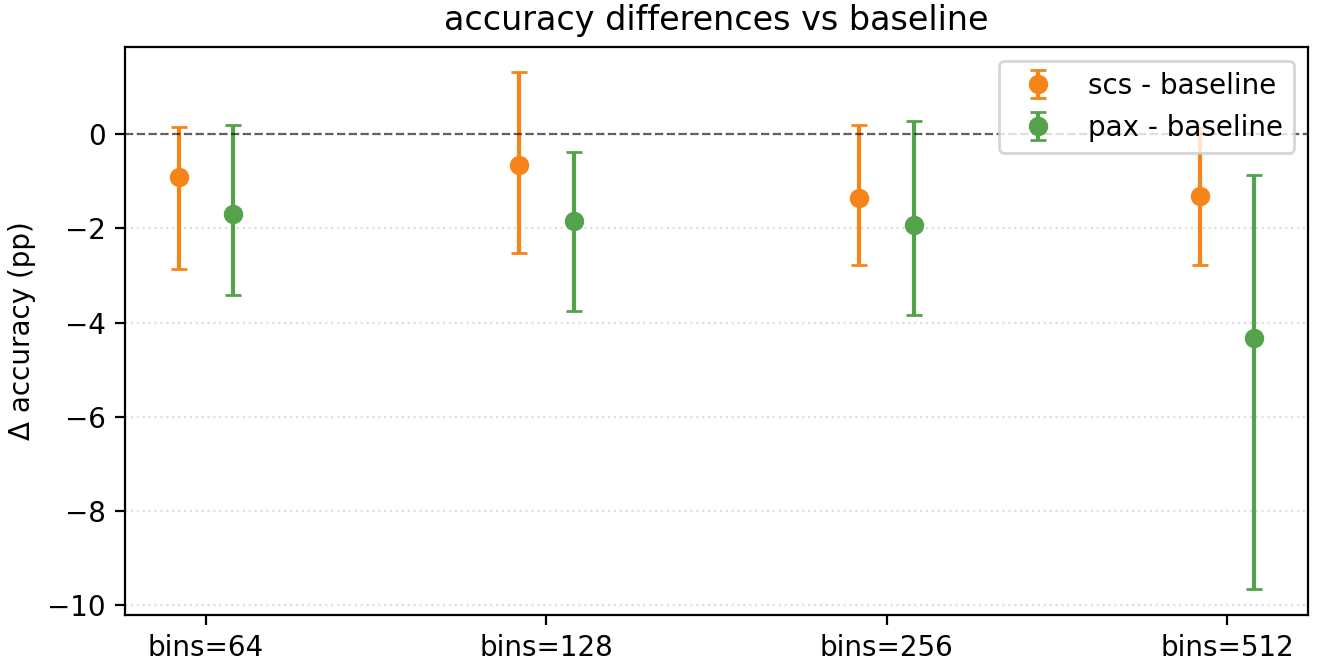}
    }
    \hfill
    \subfloat[F1.]{
        \includegraphics[width=0.48\textwidth]{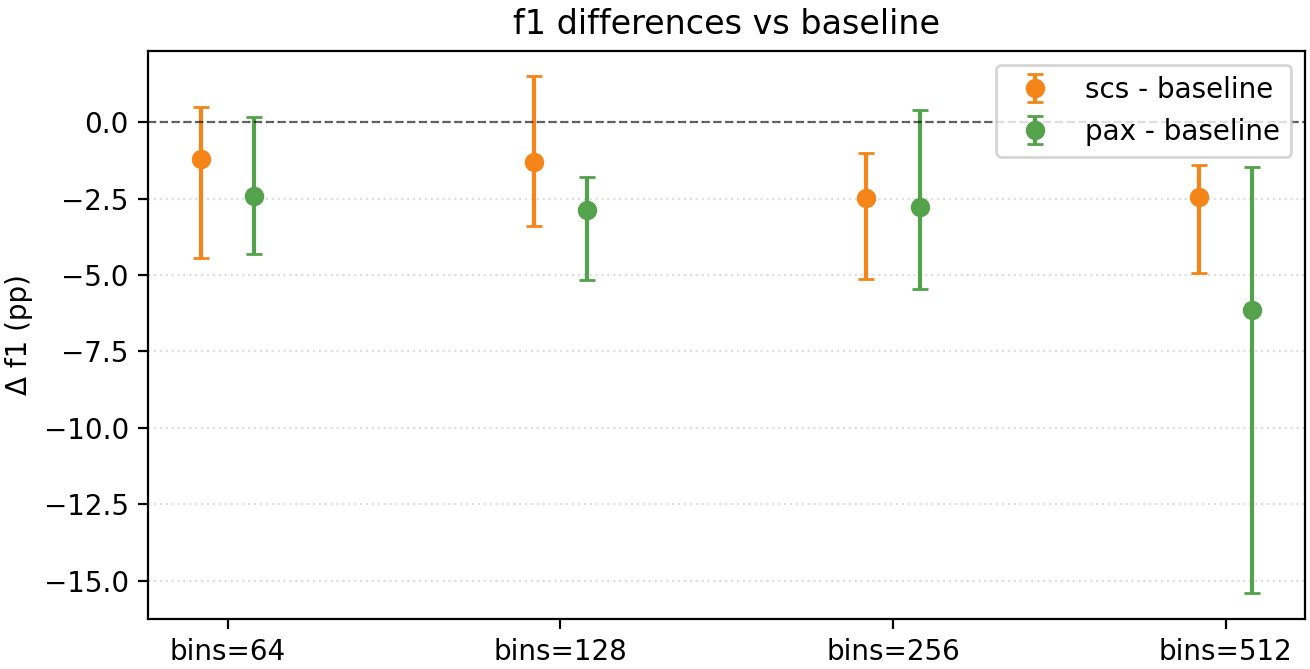}
    }
    \caption{Comparison of the performance gap between FedSCS and PAX to XGboost-baseline.}
    \label{fig:pef}
\end{figure*}







We trained XGBoost models for $10$ boosting rounds to study the effect of
histogram bin resolution on performance. Four bin sizes were evaluated: $B\in\{64,128,256,512\}$. Data were split into $80\%$ training and $20\%$ validation sets using a deterministic hash-based partitioning scheme. Pre-processing consisted of feature extraction followed by window mixing to avoid overfitting to temporally correlated segments and to preserve training heterogeneity.
Models were trained with a maximum tree depth of $4$ and hyperparameters $\lambda=1.0$, $\gamma=0.1$, and learning rate $\eta=0.2$. All $16$ classes were used, with the standard XGBoost multiclass softmax loss.

Across all evaluated metrics, FedSCS-XGB closely matches centralized XGBoost and consistently outperforms PAX under identical hyperparameter settings (Fig. ~\ref{fig:pef}, ~\ref{fig:diff_comp}, Table~\ref{tab:pax:massive_bins_metrics}).
The accuracy gap to the centralized baseline remains below 1.5\% point for all bin counts, while PAX exhibits a substantially larger degradation. FedSCS-XGB additionally shows higher and more stable F1-scores across clients, which can be seen in Fig. \ref{fig:diff_comp}. This is indicating improved robustness under class imbalance and client heterogeneity. Performance remains stable when varying the number of bins from 128 to 512, suggesting limited sensitivity to bin resolution once global alignment is
achieved.
Conversely, an increase in bin size does not necessarily result in enhanced performance metrics, as in Table \ref{tab:pax:massive_bins_metrics} indicated, which also holds true for the baseline itself.
\section{Discussion}\label{sec:discussion}
The results empirically support the convergence analysis by showing that FedSCS-XGB reaches the regime predicted by the sufficient, non-constructive proof using practical hyperparameter choices. While the analysis guarantees existence of a sketch accuracy level $\alpha$ that ensures proximity to the centralized objective, the experiments demonstrate that this regime is attainable without excessive bin counts. Compared to PAX, the improved performance and reduced client-wise variability suggest that explicit global bin alignment is more effective than party-adaptive surrogate representations in heterogeneous federated settings.
Although the evaluation is conducted on a cleanly recorded HAR dataset, these properties are particularly relevant for wearable-based monitoring of individuals with spinal cord injury, where sensor noise, subject-specific movement patterns, and physiological variability are common.
The observed robustness across clients indicates that FedSCS-XGB is well suited as extendable foundation for federated activity and vital-sign modeling under such real-world conditions. It should be noted that FedSCS-XGB does not yet support personalization. 

\section{Conclusion}\label{sec:conclusion}
We presented FedSCS-XGB, a server-centric surrogate protocol for federated XGBoost, and showed both theoretically and empirically that it can closely approximate centralized histogram XGBoost under sufficient hyperparameterization. Compared to PAX, FedSCS-XGB achieves higher accuracy and F1-scores with improved stability across clients. These results highlight global Hessian-weighted bin alignment as a key mechanism for effective federated tree-based learning.

Future work will extend the approach to noisy, longitudinal wearable datasets and personalized modeling of activity and physiological signals in real-world SCI cohorts.
\section*{Acknowledgment}

This work was supported by the Federal Ministry of Research, Technology and Space (BMFTR, Germany) as part of NeuroSys: Efficient AI-methodes for neuromorphic computing in practice (Projekt D) - under Grant 03ZU2106DA.  And by the Schweizer Paraplegiker Stiftung and the ETH Zürich Foundation (2021-HS-348) - Digital Transformation in Personalized Healthcare.
Data used in this project were conducted in compliance with relevant Swiss laws, institutional guidelines, and the Declaration of Helsinki. The study protocol was reviewed and approved by the Ethics Committee of Northwestern and Central Switzerland, EKNZ-2023-00400 (approved on 25.04.2023).

\bibliographystyle{IEEEtran}
\bibliography{literature}

@inproceedings{Bensland2023,
   author = {Sebastian Bensland and Alan Paul and Leoni Grossmann and Inge Eriks-Hogland and Robert Riener and Diego Paez-Granados},
   booktitle = {IEEE International Conference on System Integration},
   city = {Atlanta, US},
   title = {{Healthcare Monitoring for SCI individuals: Learning Activities of Daily Living through a SlowFast Network}},
   month = {1},
   year = {2023},
   url = {https://doi.org/10.1109/SII55687.2023.10039043},
}

@article{aouedi2024flsurvey,
  author       = {Oussama Aouedi and Alessandro Sacco and Lubna Umar Khan and Duc C. Nguyen and Mohsen Guizani},
  title        = {Federated Learning for Human Activity Recognition: Overview, Advances, and Challenges},
  journal      = {IEEE Open Journal of the Communications Society},
  volume       = {5},
  pages        = {7341--7367},
  year         = {2024},
  doi          = {10.1109/OJCOMS.2024.3484228}
}

@inproceedings{sozinov2018flhar,
  author       = {Konstantin Sozinov and Valeriy Vlassov and Sergey Girdzijauskas},
  title        = {Human Activity Recognition Using Federated Learning},
  booktitle    = {IEEE Intl Conf on Parallel \& Distributed Processing with Applications (ISPA)},
  year         = {2018},
  pages        = {1103--1111},
  doi          = {10.1109/BDCloud.2018.00164}
}

@inproceedings{concone2022hierarchicalfl,
  author       = {Francesco Concone and Chiara Ferdico and Giuseppe La Re and Massimo Morana},
  title        = {A Federated Learning Approach for Distributed Human Activity Recognition},
  booktitle    = {IEEE Int. Conf. on Smart Computing (SMARTCOMP)},
  year         = {2022},
  pages        = {269--274},
  doi          = {10.1109/SMARTCOMP55677.2022.00066}
}

@article{zhou2022twoDfl,
  author       = {Xinyu Zhou and Wei Liang and Jianxun Ma and Zhi Yan and Ke I.-K. Wang},
  title        = {2D Federated Learning for Personalized Human Activity Recognition in Cyber-Physical-Social Systems},
  journal      = {IEEE Internet of Things Journal},
  year         = {2022},
  doi          = {10.1109/JIOT.2022.9693094}
}

@article{ouyang2023clusterfl,
  author       = {Xiaofeng Ouyang and Zhijun Xie and Jing Zhou and Guoliang Xing and Jing Huang},
  title        = {ClusterFL: A Clustering-based Federated Learning System for Human Activity Recognition},
  journal      = {ACM Transactions on Sensor Networks},
  volume       = {19},
  number       = {1},
  pages        = {1--32},
  year         = {2023},
  doi          = {10.1145/3554980}
}

@article{chen2020fedhealth,
  author       = {Yiqiang Chen and Changhong Yu and Xiaoqing Qin and Wen Gao and Jing Wang},
  title        = {FedHealth: A Federated Transfer Learning Framework for Wearable Healthcare},
  journal      = {IEEE Intelligent Systems},
  year         = {2020},
  doi          = {10.1109/MIS.2020.9076082}
}

@article{cheng2023protohar,
  author       = {Dong Cheng and Lei Zhang and Xin Wang and Hao Wu and Aoran Song},
  title        = {ProtoHAR: Prototype-Guided Personalized Federated Learning for Human Activity Recognition},
  journal      = {IEEE Transactions on Neural Networks and Learning Systems},
  year         = {2023},
  doi          = {10.1109/TNNLS.2023.10122911}
}

@inproceedings{li2021metahar,
  author       = {Chen Li and Dapeng Niu and Bo Jiang and Xiang Zuo and Jian Yang},
  title        = {Meta-HAR: Federated Representation Learning for Human Activity Recognition},
  booktitle    = {Proceedings of the Web Conference (WWW)},
  year         = {2021},
  pages        = {912--922},
  doi          = {10.1145/3442381.3450006}
}

@article{wang2024hydra,
  author       = {Peng Wang and Tianhao Ouyang and Qi Wu and Qing Huang and Jun Gong and Xiaolin Chen},
  title        = {Hydra: Hybrid-Model Federated Learning for Human Activity Recognition on Heterogeneous Devices},
  journal      = {Journal of Systems Architecture},
  volume       = {147},
  pages        = {103052},
  year         = {2024},
  doi          = {10.1016/j.sysarc.2023.103052}
}

@inproceedings{tu2021feddl,
  author       = {Lifeng Tu and Xiaofeng Ouyang and Jing Zhou and Yuxiang He and Guoliang Xing},
  title        = {FedDL: Federated Learning via Dynamic Layer Sharing for Human Activity Recognition},
  booktitle    = {Proc. 19th ACM Conference on Embedded Networked Sensor Systems (SenSys)},
  year         = {2021},
  pages        = {15--28},
  doi          = {10.1145/3485730.3485946}
}

@inproceedings{gajanin2024asynchronous,
  author       = {R. Gajanin and A. Danilenka and A. Morichetta and S. Nastic},
  title        = {Towards Adaptive Asynchronous Federated Learning for Human Activity Recognition},
  booktitle    = {Proc. 14th Int. Conf. on the Internet of Things},
  year         = {2024},
  doi          = {10.1145/3703790.3703795}
}

@article{xiao2021harflprivacy,
  author       = {Z. Xiao and X. Xu and H. Xing and F. Song and X. Wang and B. Zhao},
  title        = {A Federated Learning System with Enhanced Feature Extraction for Human Activity Recognition},
  journal      = {Knowledge-Based Systems},
  volume       = {229},
  pages        = {107338},
  year         = {2021},
  doi          = {10.1016/j.knosys.2021.107338}
}

@article{glisic2024shc,
  author  = {Glisic, M. and colleagues},
  title   = {Changes in Secondary Health Conditions Among Individuals With Spinal Cord Injury After Transition From Inpatient Rehabilitation to Community Living},
  journal = {American Journal of Physical Medicine \& Rehabilitation},
  volume  = {103},
  number  = {11S},
  pages   = {S260},
  year    = {2024},
  month   = {Nov.},
  doi     = {10.1097/PHM.0000000000002600}
}

@inproceedings{ong2020adaptive,
  author    = {Y. S. Ong and B. Hooi and B. Q. Ho and D. S. Z. Sng and O. S. Y. Goh},
  title     = {Adaptive Histogram-Based Gradient Boosted Trees for Federated Learning},
  booktitle = {2020 IEEE International Conference on Big Data},
  year      = {2020},
  pages     = {1171--1180},
  doi       = {10.1109/BigData50022.2020.9378134}
}

@article{masson2019ddsketch,
  author  = {Masson, J. and others},
  title   = {DDSketch: A Fast and Fully-Mergeable Quantile Sketch with Relative-Error Guarantees},
  journal = {Proceedings of the VLDB Endowment},
  year    = {2019},
  volume  = {12},
  number  = {12},
  pages   = {2190--2201}
}

@misc{jones_federated_2022,
  title     = {Federated {XGBoost} on {Sample}-{Wise} {Non}-{IID} {Data}},
  url       = {http://arxiv.org/abs/2209.01340},
  doi       = {10.48550/arXiv.2209.01340},
  urldate   = {2025-07-09},
  publisher = {arXiv},
  author    = {Jones, Katelinh and Ong, Yuya Jeremy and Zhou, Yi and Baracaldo, Nathalie},
  month     = sep,
  year      = {2022},
  note      = {arXiv:2209.01340 [cs]}
}

@inproceedings{lindskog_histogram-based_2023,
  title     = {Histogram-{Based} {Federated} {XGBoost} using {Minimal} {Variance} {Sampling} for {Federated} {Tabular} {Data}},
  url       = {http://arxiv.org/abs/2405.02067},
  doi       = {10.1109/FMEC59375.2023.10306242},
  urldate   = {2025-07-11},
  booktitle = {2023 {Eighth} {International} {Conference} on {Fog} and {Mobile} {Edge} {Computing} ({FMEC})},
  author    = {Lindskog, William and Prehofer, Christian and Singh, Sarandeep},
  month     = sep,
  year      = {2023},
  note      = {arXiv:2405.02067 [cs]},
  pages     = {182--189}
}

@inproceedings{ma_gradient-less_2023,
  title     = {Gradient-less {Federated} {Gradient} {Boosting} {Trees} with {Learnable} {Learning} {Rates}},
  url       = {http://arxiv.org/abs/2304.07537},
  doi       = {10.1145/3578356.3592579},
  urldate   = {2025-07-11},
  booktitle = {Proceedings of the 3rd {Workshop} on {Machine} {Learning} and {Systems}},
  author    = {Ma, Chenyang and Qiu, Xinchi and Beutel, Daniel J. and Lane, Nicholas D.},
  month     = may,
  year      = {2023},
  note      = {arXiv:2304.07537 [cs]},
  pages     = {56--63}
}

@misc{ong_adaptive_2020,
  title     = {Adaptive {Histogram}-{Based} {Gradient} {Boosted} {Trees} for {Federated} {Learning}},
  url       = {http://arxiv.org/abs/2012.06670},
  doi       = {10.48550/arXiv.2012.06670},
  urldate   = {2025-07-14},
  publisher = {arXiv},
  author    = {Ong, Yuya Jeremy and Zhou, Yi and Baracaldo, Nathalie and Ludwig, Heiko},
  month     = dec,
  year      = {2020},
  note      = {arXiv:2012.06670 [cs]}
}

@inproceedings{bodynek_applying_2023,
  title     = {Applying {Random} {Forests} in {Federated} {Learning}: {A} {Synthesis} of {Aggregation} {Techniques}},
  shorttitle = {Applying {Random} {Forests} in {Federated} {Learning}},
  url       = {https://publikationen.bibliothek.kit.edu/1000162020},
  language  = {de},
  urldate   = {2025-07-30},
  booktitle = {Wirtschaftsinformatik 2023 {Proceedings}},
  author    = {Bodynek, Mattis and Leiser, Florian and Thiebes, Scott and Sunyaev, Ali},
  year      = {2023}
}

@inproceedings{chen2016xgboost,
  author    = {Tianqi Chen and Carlos Guestrin},
  title     = {{XGBoost}: A Scalable Tree Boosting System},
  booktitle = {Proceedings of the 22nd ACM SIGKDD International Conference on Knowledge Discovery and Data Mining},
  year      = {2016},
  pages     = {785--794},
  doi       = {10.1145/2939672.2939785}
}

@article{elgendi2025balancing,
  title   = {Balancing Cardiac Privacy with Quality in Video Recordings},
  author  = {Elgendi, Mohamed and Yu, Aojie and Bhutani, Saksham and Menon, Carlo},
  journal = {Communications Medicine},
  volume  = {5},
  number  = {486},
  year    = {2025},
  doi     = {10.1038/s43856-025-01175-0},
  publisher = {Springer Nature}
}

@techreport{vepakomma_split_2018,
	title = {Split learning for health: {Distributed} deep learning without sharing raw patient data},
	shorttitle = {Split learning for health},
	url = {http://arxiv.org/abs/1812.00564},
	doi = {10.48550/arXiv.1812.00564},
	number = {arXiv:1812.00564},
	urldate = {2022-07-28},
	institution = {arXiv},
	author = {Vepakomma, Praneeth and Gupta, Otkrist and Swedish, Tristan and Raskar, Ramesh},
	month = dec,
	year = {2018},
	note = {arXiv:1812.00564 [cs, stat]
type: article},
}

@article{kalabakov2024federated,
  title   = {Federated Learning for Activity Recognition: A System-Level Perspective},
  author  = {Kalabakov, Stefan and Jovanovski, Borche and Rakovic, Valentin and
             Denkovski, Daniel and Pfitzner, Bjarne and Konak, Orhan and
             Arnrich, Bert and Gjoreski, Hristijan},
  journal = {IEEE Access},
  year    = {2024},
  doi     = {10.1109/ACCESS.2023.3289220}
}

@article{bresnahan2022pain,
	title = {The demographics of pain after spinal cord injury: a survey of our model system},
	volume = {8},
	copyright = {2022 The Author(s), under exclusive licence to International Spinal Cord Society},
	issn = {2058-6124},
	shorttitle = {The demographics of pain after spinal cord injury},
	url = {https://www.nature.com/articles/s41394-022-00482-1},
	doi = {10.1038/s41394-022-00482-1},
	language = {en},
	number = {1},
	urldate = {2025-11-12},
	journal = {Spinal Cord Series and Cases},
	publisher = {Nature Publishing Group},
	author = {Bresnahan, James J. and Scoblionko, Benjamin R. and Zorn, Devon and Graves, Daniel E. and Viscusi, Eugene R.},
	month = jan,
	year = {2022},
	pages = {14},

}

@book{molnar2019interpretable,
  title      = {Interpretable Machine Learning},
  author     = {Molnar, Christoph},
  year       = {2019},
  subtitle   = {A Guide for Making Black Box Models Explainable},
  note       = {Online book, Chapter~9: Interpretable Models -- Decision Trees},
  url        = {https://christophm.github.io/interpretable-ml-book/}
}

@article{friedman2000greedy,
	title = {Greedy {Function} {Approximation}: {A} {Gradient} {Boosting} {Machine}},
	volume = {29},
	shorttitle = {Greedy {Function} {Approximation}},
	doi = {10.1214/aos/1013203451},
		journal = {The Annals of Statistics},
	author = {Friedman, Jerome},
	month = nov,
	year = {2000},
}

@misc{beutel_flower_2022,
	title = {Flower: {A} {Friendly} {Federated} {Learning} {Research} {Framework}},
	shorttitle = {Flower},
	url = {http://arxiv.org/abs/2007.14390},
	doi = {10.48550/arXiv.2007.14390},
	urldate = {2026-01-20},
	publisher = {arXiv},
	author = {Beutel, Daniel J. and Topal, Taner and Mathur, Akhil and Qiu, Xinchi and Fernandez-Marques, Javier and Gao, Yan and Sani, Lorenzo and Li, Kwing Hei and Parcollet, Titouan and Gusmão, Pedro Porto Buarque de and Lane, Nicholas D.},
	month = mar,
	year = {2022},
	note = {arXiv:2007.14390 [cs]},
}

@inproceedings{Ejtehadi2023LearningRehabilitation,
    title = {{Learning Activities of Daily Living from Unobtrusive Multimodal Wearables: Towards Monitoring Outpatient Rehabilitation}},
    year = {2023},
    booktitle = {IEEE International Conference on Rehablitation Robotics (ICORR)},
    author = {Ejtehadi, Mehdi and Amrein, Sabrina and Eriks Hoogland, Inge and Riener, Robert and Paez-Granados, Diego},
    month = {9},
    pages = {1--7},
    url = {https://www.research-collection.ethz.ch/handle/20.500.11850/619499},
    organization = {IEEE},
    address = {Singapore},
    arxivId = {10.3929/ ethz-b-000619499}
}

@INPROCEEDINGS{tifexpy,
  author={Ejtehadi, Mehdi and Graham, Gloria Edumaba and Ringstrom, Cailin and Du, Elisa and Riener, Robert and Paez-Granados, Diego},
  booktitle={2025 International Conference On Rehabilitation Robotics (ICORR)}, 
  title={Tifex-Py: Time-Series Feature Extraction for Python in a Human Activity Recognition Scenario}, 
  year={2025},
  volume={},
  number={},
  pages={1332-1339},
  doi={10.1109/ICORR66766.2025.11062978}}

\end{document}